\documentclass[conference,10pt,letterpaper]{IEEEtran}
\usepackage{cite}

\ifCLASSINFOpdf
\else
\fi

\ifCLASSINFOpdf
 \usepackage[pdftex]{graphicx}
\else
\fi
\usepackage[cmex10]{amsmath}
\usepackage{array}
\usepackage{mdwmath}
\usepackage{mdwtab}
\usepackage{eqparbox}

\newcommand{\llan}{\left\langle\underline }
\newcommand{\rlan}{\right\rangle}

\newcommand{\pc}{\mathbf{P}}

\usepackage{graphicx}
\usepackage{amssymb}
\usepackage{amsthm}
\usepackage{subfig}
\usepackage{latexsym}
\usepackage{inputenc}
\usepackage{url}
\usepackage{euler}
\usepackage{amsmath}
\usepackage{amssymb}
\usepackage[inline]{asymptote}
\usepackage{proof}
\usepackage{tabularx}

\newcommand{\bbb}{\mathbb{B}}
\newcommand{\mbR}{\mathbb{R}}

\newtheorem{theorem}{Theorem}
\newtheorem{definition}{Definition}
\newtheorem{proposition}{Proposition}

\begin{document}

\title{Contamination-Free Measures and Algebraic Operations}

\author{\IEEEauthorblockN{A. Mani}
\IEEEauthorblockA{Department of Pure Mathematics\\
University of Calcutta\\
9/1B, Jatin Bagchi Road\\
Kolkata(Calcutta)-700029, India\\
Email: {a.mani.cms@gmail.com}\\
Homepage: \url{http://www.logicamani.in}}}


\maketitle

\begin{abstract}
An open concept of rough evolution and an axiomatic approach to granules was also developed in \cite{AM240} by the present author. Subsequently the concepts were used in the formal framework of rough Y-systems (\textsf{RYS}) for developing on granular correspondences in \cite{AM1800}. These have since been used for a new approach towards comparison of rough algebraic semantics across different semantic domains by way of correspondences that preserve rough evolution and try to avoid contamination. In this research paper, we propose methods and semantics for handling possibly contaminated operations and structured bigness. These would also be of natural interest for relative consistency of one collection of knowledge relative other.
\end{abstract}
\textbf{Keywords}: \begin{small}{Contaminated Operations, Rough Measures, Granular Axioms, SNC, Algebraic Semantics, Growth-Like Functions.}\end{small}


\section{Introduction}

In the present author's perspective of the contamination problem and axiomatic granular approach \cite{AM240}, it is required that reasonable measures should \emph{carry information} about the underlying \emph{rough evolution}. A contaminated operation is simply an operation (used in the particular rough semantics in question) without reasonable semantic justification in the domain under consideration. Such operations can be particularly problematic when the semantics is intended for modeling vague reasoning. Some approaches using new dialectic counting strategies are developed and related semantic structures have been developed in the same paper. It is not easy to apply it in practice in its presented form in \cite{AM240}. So another approach that reduces \emph{measurability} to \emph{relative comparability with preservation of rough evolution} has been proposed and developed in \cite{AM978}. Thus for example, the intended use of measures of rough inclusion may be reduced to comparison of related formulas. In this research paper, we extend the approach to deal with possibly contaminated operations and algebraically deal with bigness/relevance. 

For simplicity, we will work with special kinds of \textsf{RYS} that are partially ordered partial algebras. Let $X= \left\langle \underline{X}, \pc , (l_{i}^{*})_{1}^{n}, (u_{i}^{*})_{1}^{n}, \oplus, \odot, 1 \right\rangle$ and $Y,= \left\langle \underline{Y}, \pc_{2} , (l_{i})_{1}^{n}, (u_{i})_{1}^{n}, \oplus_{2}, \odot_{2}, \Finv \right\rangle$ be two general rough Y-systems (\textsf{RYS}) with associated granulations $\mathcal{G}$ and $\mathcal{G}_{2}$ respectively. The interpretation of $\underline{X}$ and $\underline{Y}$ will be that of some collections of objects of interest and not necessarily of rough objects. Granules need not be rough objects in general and this aspect affects the way we express the semantics. Any map $\varphi : X  \longmapsto   Y$ will be taken to be a \emph{correspondence} between the systems, though of course only those that preserve granularity or approximations in some sense would be of interest. It is also possible to adapt other particular approximation frameworks such as the abstract approximation space framework of \cite{CD3} for the purposes of the present paper in a formal mathematical setting.

\subsection{Background}

An adaptation of the precision-based \emph{classical granular computing paradigm} to rough sets is explained in \cite{Ya01,TYL}.
The axiomatic approach to granularity initiated in \cite{AM99} has been developed by the present author in the direction of contamination reduction in \cite{AM240}. From the order-theoretic/algebraic point of view, the deviation is in a very new direction relative the precision-based paradigm. The paradigm shift includes a new approach to measures and this is taken up in a new direction in this research paper.

We expect the reader to be aware of the different measures used in \textsf{RST} like those of degree of rough inclusion, rough membership (see \cite{AG3} and references therein), $\alpha$-cover and consistency degrees of knowledges. If these are not representable in terms of granules through term operations formed from the basic ones \cite{AM99}, then they are not truly functions/degrees of the rough domain. In \cite{AM240}, such measures are said to be \emph{non-compliant for the rough context} in question and new granular measures have been proposed as replacement of the same. 

Knowledge of partial algebras (see \cite{BU}), weak equalities and closed morphisms will be assumed.
In a partial algebra, for term functions $p, q$, the \emph{weak} equality is defined
via, $p\stackrel{\omega}{=}q\; \mathrm{iff} \;(\forall x \in dom(p)\cap dom(q)) 
p(x)=q(x)$. The \emph{weak-strong} equality is defined via, 
$dom(p)= dom(q), \&  p\stackrel{\omega^{*}}{=}q\; \mathrm{iff} \;(\forall
x \in dom(p))  p(x)=q(x)$. By a $\oplus$ -morphism $f$ between two algebras $X$ and $Y$ with interpretations for the operation symbol $\oplus$, we will mean a map that preserves the interpretation of $\oplus$, that is  $(\forall x, y  \in  X )   f(x \oplus  y)= f(x) \oplus  f(y)$. In other words it is a forgetful morphism that preserves the interpretation of the operation $\oplus$.   

We use the simplified approach to \textsf{RYS} of \cite{AM1800} (instead of the version in \cite{AM240}) that avoids the operator $\iota$ and is focused on a general set-theoretic perspective. These structures are provided with enough structure so that a Meta-C and at least one Meta-R of roughly equivalent objects along with admissible operations and
predicates are associable. For the language, axioms and notation see \cite{AM1800}. Our models will assume total operations as in \cite{AM1800}. \emph{Admissible granulations} \cite{AM240} will be those granulations satisfying
the conditions \textsf{WRA, LFU, LS}.
\section{SNC and Variations}

In this section we update some of the material of \cite{AM1800} with complete proofs and introduce important modifications of the concept of a \textsf{SNC}. A map from a \textsf{RYS} $S_{1}$ to another $S_{2}$ will be referred to as a
\emph{correspondence}. It will be called a \emph{morphism}
if and only if it preserves the operations $\oplus$ and $\odot$. We will also speak
of $\oplus$-morphisms and $\odot$ -morphisms if the correspondence preserves just
one of the partial/total operations. By Sub-Natural
Correspondences (\textsf{SNC}), we seek to capture simpler correspondences that associate
granules with elements representable by granules and do not necessarily commit the context to Galois
connections. An issue with \textsf{SNC}s is that it fails to adequately capture granule centric correspondences that may violate the injectivity constraint and may not play well with morphisms.
 
\begin{definition}
Let If $S_{1}$ and $S_{2}$ are two \textsf{RYS} with granulations,
$\mathcal{G}_1$ and $\mathcal{G}_2$ respectively, consisting of successor neighborhoods or
neighborhoods. A correspondence $\varphi:  S_1 \longmapsto S_2 $ will be said to be a \emph{Proto Natural Correspondence (PON)} (respectively \emph{Pre-Natural Correspondence (PNC)}) iff the second (respectively both) of the following conditions hold:
\begin{enumerate}
\item {$\varphi_{|\mathcal{G}_{1}}$ is injective $:  \mathcal{G}_{1} \longmapsto   \mathcal{G}_{2}$.}
\item {there is a term function $t$ in the signature of $S_2$ such that 
$(\forall [x] \in \mathcal{G}_1)(\exists y_{1},\ldots y_{n} \in \mathcal{G}_2)
  \varphi ([x])=t (y_{1}, \ldots , y_{n})$.}
\item {the $y_i$s in the second condition are generated by $\varphi(\{x\})$ for each $i$ ($\{x\}$ being a singleton).}
\end{enumerate}
An injective correspondence $\varphi:  S_1 \longmapsto S_2 $ will be said to be a \emph{SNC} iff the last two conditions hold.
\end{definition}

Note that the base sets of \textsf{RYS} may be semi-algebras of sets. 

\begin{theorem}
If $\varphi$ is a \textsf{SNC} and both $\mathcal{G}_1$ and
$\mathcal{G}_2$
are partitions, then the non-trivial cases should be equivalent to one of the
following:
\begin{itemize}
\item {\textsf{B1:} $(\forall \{x\}\in S_{1})   \varphi ([x]) = [\varphi
(\{x\})]$. \textsf{B2:} $(\forall \{x\}\in S_{1})   \varphi ([x]) = \sim [\varphi
(\{x\})]$.}
\item {\textsf{B3:} $(\forall \{x\}\in S_{1})   \varphi ([x]) = \bigcup_{y\in
[x]} [\varphi(\{y\})]$. \textsf{B4:} $(\forall \{x\}\in S_{1})   \varphi ([x]) =
\sim (\bigcup_{y\in [x]} [\varphi
(\{y\})])$.}
\end{itemize}
\end{theorem}
\begin{proof}
Intersection of two distinct classes is always empty. 
If $\sim$ is defined, then the second and fourth case will be possible.
So these four exhaust all possibilities.  
\end{proof}

\begin{theorem}
If $\varphi$ is a \textsf{SNC}, $\mathcal{G}_1$ is a partition
and $\mathcal{G}_2$ is a system of blocks, then the non-trivial
cases should be equivalent to one of the following ($\sim (\beta (x))$ is the set $\{\sim
B  : B\in \beta (x)\}$ and $\{x \}\in S_{1}$ ):
\begin{description}
\item [C1:]{ $   \varphi ([x]) = \cup \beta (\varphi
(\{x\})).$}
\item [C2:]{$   \varphi ([x]) = \sim (\cup \beta
(\varphi (\{x\})))$.}
\item [C3:]{ $   \varphi ([x]) = \cap \beta (\varphi
(\{x\}))$.}
\item [C4:]{$   \varphi ([x]) = \sim (\cap
\beta (\varphi (\{x\})))$.}
\item [C5:]{$   \varphi ([x]) = \bigcup_{y\in
[x]} \cup \beta (\varphi (\{x\}))$.}
\item [C6:]{$   \varphi
([x]) = \sim (\bigcup_{y\in [x]} \cup \beta (\varphi (\{x\})))$.}
\item [C7:]{$   \varphi ([x]) = \bigcap_{y\in
[x]} \cup \beta (\varphi (\{x\}))$.}
\item [C8:]{$   \varphi
([x]) = \sim (\bigcup_{y\in [x]} \cap \beta (\varphi (\{x\})))$.}
\end{description}

\end{theorem}

\begin{theorem}
If we take $S_{1}$ to be a classical RST-RYS and $S_{2}$ is a TAS-RYS with approximations
$l\mathcal{B}^{*}$ and $u\mathcal{B}^{*}$ and $\varphi$ is a \textsf{SNC} and a
$\oplus$-morphism satisfying the first condition above, then all of the following hold:
\begin{enumerate}
\item {$\varphi (x^{l}) \subseteq   (\varphi (x))^{l\mathcal{B}^{*}}$,}
\item {$\varphi (x^{u}) \subseteq   (\varphi (x))^{u\mathcal{B}^{*}}$, }
\item {If $\varphi$ is a morphism, that preserves $\emptyset $ and $1$, then equality
holds in the above two statements.}
\end{enumerate}
But the converse need not hold in general. 
\end{theorem}

\begin{proof}
\begin{enumerate}
\item {If $A\in S_{1}$, then $\varphi (A^{l})= \varphi (\bigcup_{[\{x\}]\subseteq
A} [\{x\}])=\bigcup_{[\{x\}]\subseteq A} \varphi ([\{x\}])=
\bigcup_{[\{x\}]\subseteq A} \cap \beta(\{\varphi (\{x\})\}),$ and that is a subset of $
\bigcup_{\varphi([\{x\}])\subseteq \varphi(A)} \cap \beta(\{\varphi (\{x\})\})$.
Some of the $\mathcal{B}^{*}$ elements included in $\varphi (A)$ may be lost if we start
from $\varphi (A^l)$.}
\item {If $A\in S_{1}$, then $\varphi (A^{u}) =  \varphi (\bigcup_{[\{x\}]\cap
A \neq \emptyset} [\{x\}]) = \bigcup_{[\{x\}]\cap
A \neq \emptyset} \varphi ([\{x\}]) =
\bigcup_{[\{x\}]\cap
A \neq \emptyset} \cap \beta(\{\varphi (\{x\})\})$, and that is a subset of $ \bigcup_{y
\cap  \varphi(A) \neq \emptyset} \cap \beta(y).$
In the last part possible values of $y$ include all of the values in $\varphi (A)$.}
\item {Because of the conditions on $\varphi$, for any $A, B\in S_1$ if $A\cap B =
\emptyset$, then $\varphi (A)\cap \varphi(B) = \emptyset $. So a definite element must be
mapped into a union of disjoint granules in $S_2$. Further, for $ A\in S_{1}$ and $\xi$,
$\eta$, $\zeta$ being abbreviations for $(\cap\beta(\varphi (x)))\cap \varphi(A)\neq
\emptyset$,  $\varphi ([\{x\}]\cap A)\neq \emptyset$ and $\varphi ([\{x\}])\cap
\varphi(A)\neq\emptyset$ respectively, 
$(\varphi(A))^{u\mathcal{B}^{*}} = \bigcup_{\xi}\cap\beta(\varphi (x)) =
\varphi(\bigcup_{\xi} [\{x\}]) = \varphi(\bigcup_{\zeta} [\{x\}]) = \varphi (\bigcup
_{\eta} [\{x\}])$, which is $\varphi (A^u)).$ }
\end{enumerate}
\end{proof}

\begin{theorem}
If we take $S_{1}$ to be a classical RST-RYS and $S_{2}$ as a TAS-RYS with approximations
$l\mathcal{T}$ and $u\mathcal{T}$ and $\varphi$ is a \textsf{SNC} and a $\oplus$ -
morphism satisfying for each singleton $\{x\}\in S_1$,  $\varphi ([\{x\}])= [\varphi
(x)]$, then all of
the following hold:
\begin{enumerate}
\item {$\varphi (x^{l}) \subseteq   (\varphi (x))^{l\mathcal{T}}$,}
\item {$\varphi (x^{u}) \subseteq   (\varphi (x))^{u\mathcal{T}}$. }
\end{enumerate}
\end{theorem}

\section{Comparable Correspondences}

Growth functions are well known in summability, numerical analysis and computer science, but are generally presented in a simplistic way in most of the literature. In \cite{AM978}, these are presented in a more mature form and related to rough sets over the reals. Key higher order similarities exist between such concepts of \emph{comparable over sufficiently large domains} in the theory and the idea of comparability in this paper.

The main steps of the comparison approach in this research paper consist in specifying the semantic domains of interest, formulation the two or more granular semantics as a \textsf{RYS} (a formal language is not absolutely essential), specification of the granular rough evolutions of interest, identification of the granular correspondences of interest, computation of the comparative status of the granular correspondences and finally augmentation of the best correspondences with reasonable measures if sensible. 

Let $X,  Y$ be two general rough Y-systems (\textsf{RYS}) with associated granulations $\mathcal{G}$ and $\mathcal{G}_{2}$ respectively as in the introduction. The interpretation of $\underline{X}$ and $\underline{Y}$ will be that of some collections of objects of interest and not necessarily of rough objects. Granules need not be rough objects in general and this aspect affects the way we express the semantics. Here by \emph{rough evolution} we mean the granular properties expressed by the sentences satisfied in the models.

\begin{definition}
 The atoms and coatoms of $X$ will be denoted by $A(X)$ and $  CA(X)$ respectively. $X\setminus (A(X)\cup\{0,1\})$, $X\setminus (CA(X)\cup\{0,1\})$ and $X\setminus (A(X)\cup CA(X)\cup\{0,1\})$ respectively will be denoted by $X_{a},  X_{c}$ and $ X_{ac}$ respectively. In all this, if the least element $0$ is not present in $X$ then the operation of subtracting it from $X$ will not have any effect. The key objects in this perspective may be \emph{objects relevant for the rough evolution} that may fail to be things like $X_{ac},  X_{a}$ or $X_{c}$. These will be subsets of $X$, denoted by $X_{r}$. 
\end{definition}
\begin{definition}
Let $h,  f$ be correspondences $:  X \longmapsto  Y$, then $f$ will be $\Theta_{lu}$\emph{-related} to $h$ iff for some $i$,  
$(\exists z_{0}\in X_{c} )(\forall z\in \{x ; \pc_{2} z_{0}x\} )  \pc_{2} (h(z))^{l_{i}} f(z)  \&  \pc_{2} f(z)(h(z))^{u_{i}}$.  
In contrast $f$ will be $\Theta_{uu}$\emph{-related} to $h$ iff for some $i, j$,  
\[(\exists z_{0}\in X_{c} )(\forall z\in \{x ; \pc_{2} z_{0}x\})   \pc_{2} (h(z))^{u_{i}} f(z)  \&  \pc_{2} f(z)(h(z))^{u_{j}}\]  
We will also denote the set of elements $\Theta_{lu}$ and $\Theta_{uu}$ -related to $h$ respectively by $\Theta_{lu} (h)$ and 
$\Theta_{uu} (h)$ respectively.
\end{definition}

\begin{definition}
Let $h,  f$ be correspondences $:  X \longmapsto  Y$, then $f$ will be $\Omega_{l}$\emph{-related} to $h$ iff for some $i$, $(\exists z_{0}\in X_{c} )(\forall z\in \{x ; \pc_{2} z_{0}x\} )  \pc_{2} (h(z))^{l_{i}} f(z)$.
In contrast $f$ will be $\Omega_{u}$\emph{-related} to $h$ iff for some $i$,  \[(\exists z_{0}\in X_{c} )(\forall z\in \{x ; \pc_{2} z_{0}x\})   \pc_{2} (h(z))^{u_{i}} f(z).\]
We will also write, $\Omega_{l} (h)$ and $\Omega_{u} (h)$ respectively for the set of elements $\Omega_{lu}$ and $\Omega_{uu}$ related  to $h$ respectively.
\end{definition}

\begin{definition}
Let $h,  f$ be correspondences $:  X \longmapsto  Y$, then $f$ will be $\mathbf{O}_{u}$\emph{-related} to $h$ iff for some $i$,  
\[(\exists z_{0}\in X_{c} )(\forall z\in \{x ; \pc_{2} z_{0}x\} )  \pc_{2} f(z)(h(z))^{u_{i}}.  \] 
In contrast $f$ will be $\mathbf{O}_{l}$\emph{-related} to $h$ iff for some $i$,  
\[(\exists z_{0}\in X_{c} )(\forall z\in \{x ; \pc_{2} z_{0}x\})   \pc_{2} f(z)(h(z))^{l_{i}}. \] 
We will also write, $\mathbf{O}_{l} (h)$ and $\mathbf{O}_{u} (h)$ respectively for the sets of elements $\mathbf{O}_{l}$ and $\mathbf{O}_{u}$ -related to $h$.
\end{definition}

\begin{proposition}
If $f \in  \Theta_{lu}(h)$, then it is not necessary that $h \in   \Theta_{lu}(f)$.
The result holds even when $f, h$ are morphisms. 
\end{proposition}
Examples for this can be constructed for classical rough set theory itself. This motivates the following definition.

\begin{definition}
$f$ will be \emph{symmetrically $\Theta_{lu} $-related} to $h$ iff  $f \in  \Theta_{lu}(h)$ and $h \in  \Theta_{lu}(f)$.
Further we will denote $\Theta_{lu} (h) \cap  Mor(X,Y)$ by $\mu\Theta_{lu} (h)$ when $h\in Mor(X,Y)$ and $\Theta_{lu} (h) \cap  Mor_{c}(X,Y)$ by $\mu_{c}\Theta_{lu} (h)$ when $h\in Mor_{c}(X,Y)$. $Mor_{c}(X,Y)$ being the set of closed morphisms. 
Analogously all other notions defined above can be extended.
\end{definition}

The basic idea of the above definitions is that for some sub-collections of objects, $f$ and $h$ transform objects in a similar way. To really make this useful, we need to impose structural constraints on the map (like preservation of rough evolution). Without those, the following properties will hold:

\begin{proposition}
For $f, g  \in \mu\Theta_{lu} (h) $, the following operations are well defined on $\mu\Theta_{lu} (h)$ and $\mu_{c}\Theta_{lu} (h)$:
\begin{enumerate}
\item {\[(\forall x\in X)(f+g)(x)= \left\lbrace  \begin{array}{ll}
 f(x) \oplus_{2}  g(x) & \mathrm{if  \: defined,} \\
 \mathrm{undefined} & \mathrm{otherwise.} \\
  \end{array} \right. \] }
\item {\[(\forall x\in X)(f \cdot g)(x)= \left\lbrace  \begin{array}{ll}
 f(x) \odot_{2}  g(x) & \mathrm{if\:   defined,} \\
 \mathrm{undefined} & \mathrm{otherwise.} \\
  \end{array} \right. \] }
\item {$(\forall x\in X)\imath (x)= \Finv$}
\end{enumerate}
\end{proposition}
\begin{proof}
Since $\forall a,  b  \in   Y$, $a\oplus_{2} b$ is defined or undefined and similarly for $a\odot_{2} b$, the operations are well defined morphisms. In case $Y$ is a total algebraic system, all we need to do is to verify the morphism conditions (when $X$ is a partial/total algebraic system). 
\end{proof}

\begin{definition}
Further, we can define parthood relations $\leq$ and $\leq_{c}$ respectively on $\mu\Theta_{lu} (h)$ and $\mu_{c}\Theta_{lu} (h)$ respectively as below:
$f \leq  g $ iff $(\forall x \in X)  \pc_{2}f(x) g(x) $. 
\end{definition}

\begin{proposition}
The parthood relations $\leq$ and $\leq_{c}$ are quasi-orders. $\leq$ induces a partial order on the quotient $\mu\Theta_{lu} (h)|\approx$ (and $\mu_{c}\Theta_{lu} (h)|\approx_{c}$ respectively) defined by \[f\approx g \; \mathrm{iff}\; f\leq g   \&  g\leq f.\]  
\end{proposition}

\subsection{Relevant Types of Subsets}

In the above considerations, the concept of comparison assumes that \emph{two correspondences are comparable provided they are comparable over specific types of sets}. Further the idea of \emph{specific type of sets} is restricted to ones definable by excluding atoms and co-atoms. This is not necessarily the best thing to do. The concepts of partial reducts, $\alpha$-covers and related ones  use number based exclusion criteria along with the difficulties associated with them. We propose theoretical improvements at the granulation level only and without algorithms (for now) for improving the situation.

A relevant set can be one that has a sufficiently \emph{large} subset. This means that in most algebraic approaches to semantics of rough objects, we can associate nice structures with them. We show this for the classical RST contexts later in this paper. In all cases, we can almost certainly improve the semantics through possibly definable predicates for relevance. 

Let $\bbb x$ denote the statement '$x$ is big/relevant' and $x_{o}$ be a relevant object then possible axioms for relevance/bigness may be formed from combination of axioms from below or from similar ones :

\begin{itemize}
\item {$\Delta 1$: $\bbb x$ iff $(y)(\pc x_{o} y   \longrightarrow  \pc y x^{l})$. }
\item {$\Delta_2$: $\pc x_{o}x  \&   \pc x_{o}^{l} x^{l}$.}
\item {$\Delta_3$: $\pc x_{o} x^{l}   \&   x_{o}\in \delta_{l}(S)$.}
\item {$\Delta_4$: $\pc x_{o} x^{l}   \&   x_{o}\in \delta_{lu}(S)$.}
\item {$\Delta_5$: $\pc x_{o} x^{l}   \&   x_{o}\in \delta_{u}(S)$.}
\end{itemize}
In many practical situations, relevance can be defined by feature sets and $x_{o}$ may also be an abstract object corresponding to a Boolean combination of features.  

Given such a $\bbb$ predicate, we can define a concept of '$f$ being of the \emph{$\bbb$-order of rough growth} of $g$ (in symbols $\Gamma f g$)' by 
\[(\forall y)(\forall x) (\pc x y^{l} \& \bbb x \longrightarrow \pc (fy)^{l}(gy)  \&  \pc (gy)(fy)^{u} ).\]
This is one of the possible generalizations of the concepts introduced earlier for classification. An abstract view of possible $\bbb$ axioms is in order for the following sections:

\begin{itemize}
\item {B1: $(\forall x, a)(\bbb a, \&  \pc a x  \longrightarrow  \bbb x  )$.}
\item {B2: $(\forall x)(\bbb x   \longrightarrow  \bbb x^{u})$.}
\item {B3: $(\forall x, a, b)(\bbb x  \&  \pc x a   \&  \pc a b  \longrightarrow  \bbb b)$.}
\item {BC1: $(\forall a, b)(\bbb a   \&  \bbb b   \longrightarrow  \bbb a\oplus b )$.}
\item {BC2: $(\forall a, b)(\bbb a   \&  \bbb a\oplus b   \longrightarrow  \bbb b)$.}
\item {BC3: $(\forall a, b)(\bbb a    \longrightarrow  \bbb a\oplus a )$.}
\item {BC4: $(\forall a, b )(\bbb a  \&  \bbb b   \longrightarrow  \bbb a\odot b)$.}
\item {BC5: $(\forall a, b )(\bbb a \oplus a   \longrightarrow  \bbb a)$.}
\item {BC6: $(\forall a, b, c )(\bbb b  \&  \pc a b  \& \pc b c \longrightarrow  \bbb b\odot c)$}
\end{itemize}

\begin{proposition}
In a \textsf{RYS}, all of the following hold:
\begin{itemize}
\item {If B1 holds then B2 follows.}
\item {B1 follows from B3, but the converse need not hold.}
\end{itemize}
\end{proposition}

\begin{flushleft}
\textbf{Example}\\ 
\end{flushleft}

Concepts of 'big/relevant-enough for a particular action' to be performed are fairly routine in system administration contexts. Suppose the policy is to provide additional privileges for users with specific kinds of usage patterns of resources on the Internet and possibly local repositories or cache. This policy can be implemented as a rough set computation based rule system: each user on the local network (or ISP's network) can be associated with dynamically constructed approximations of their usage. These approximations may be mapped for comparison with an abstract rough set based usage system and then the policies may be implemented. Much of the computation in this regard can be highly nontrivial, but complexity is likely to reach a plateau with increase in number of users in the system.

\section{Granular Rough Evolution}

If approximations evolve in a similar way in two different rough semantics of different contexts, then the corresponding approximations may be compared relative the other. We exactify the concept of \emph{similar way} in this section. $X,  Y$ will be \textsf{RYS} with associated granulations or equivalently inner \textsf{RYS} in all of this section. Let the set of granular axioms satisfied by $X$ and $Y$ be $A_{X}$ and $A_{Y}$ respectively.

\begin{definition}
$X$ will be said to be \emph{of strongly similar rough evolution} (SSE) as $Y$ iff all of the following hold:
\begin{itemize}
\item {Granular Inclusion: $C(A_{X}) \subseteq  C(A_{Y})$, i.e.  set of granular axioms satisfied by $X$ is included in the set of granular axioms satisfied by $Y$. }
\item {Admissibility: $\mathcal{G}_{X}, \mathcal{G}_{Y} $ are both admissible.}
\item {Equi-representability: $X$ and $Y$ have equal number of approximation operators and corresponding approximations in $X$ and $Y$ are represented by similar terms and formulas in terms of granules.}
\end{itemize}
If instead the first and second condition hold, then $X$ will be said to be of \emph{similar rough evolution} as $Y$. If the first and third alone hold, then $X$ will be said to be of \emph{sub-similar rough evolution} as $Y$. If the first alone holds, then $X$ will be said to be of \emph{psubmilar rough evolution} as $Y$. If the second and third alone hold, then $X$ will be said to be of \emph{pseudo-similar rough evolution} as $Y$.
\end{definition}

Examples for these can be had by pairing different types of formal versions of semantics explicitly described in \cite{AM240}. Note that we do not require any explicit correspondence between the associated \textsf{RYS} in the first and second conditions, but some concept of correspondence of signatures is implicit in the third. The requirement of equal number of approximations can be made redundant by expanding signatures suitably -- additional symbols for approximations being interpreted as duplicates.   

\begin{proposition}
On the class of inner \textsf{RYS} $\mathfrak{I}_{RYS}$, pseudo-similarity is an equivalence relation, while psubmilarity, subsimilarity and similarities are quasi-order relations.
\end{proposition}

\section{Comparing Two Rough Set Theories}

In classical \textsf{RST-RYS}, we know that all of the granular axioms \textsf{RA, ACG, MER, FU, NO, PS, ST, I}  hold. In a large subclass of \textsf{RST}s some consequences of these hold. Subject to admissibility of the granulations and the further restriction of equirepresentability, it is possible to compare correspondences sensibly with classical \textsf{RST-RYS}. But strong similarity remains a quasi-order relation. We investigate sub-natural correspondences in the contexts considered in \cite{AM1800} below.

The restriction to SNC means that we restrict attention to successor neighborhoods or neighborhoods. Such granules can fail to be definite elements in general (as in generalized transitive \textsf{RST} see \cite{AM270}) and maybe definite and have other properties when approximations are 'suitable'. 

\begin{definition}
Let $S_{1}$ and $S_{2}$ are two \textsf{RYS} with granulations,
$\mathcal{G}_1$ and
$\mathcal{G}_2$ respectively, consisting of successor neighborhoods or
neighborhoods and $S_{1}$ is of strongly similar rough evolution as $S_{2}$.
Any sub-natural correspondence $\varphi:  S_1 \longmapsto S_2 $ will be said to be 
\emph{smooth relative the approximations} $l,  u$ iff for each definite element $x$ relative $l,   u$ there exists a definite element $z$ relative some $l_{i},  u_{i}$ in $S_{2}$ such that $\varphi (x)= z$.
Analogously the concept of \emph{smooth pre-natural} correspondence can be defined.  
\end{definition}

Let $SNC(S_{1}, S_{2})$, $SNC_{s}(S_{1}, S_{2})$, $SM(S_{1},  S_{2})$ and  $SM_{s}(S_{1},  S_{2})$ respectively denote the set of SNCs, smooth SNCs, SNCs that are also $\oplus$-morphisms and smooth SNCs that are also $\oplus$-morphisms respectively. The corresponding concepts for pre-natural correspondences will be denoted by $PNC(S_{1}, S_{2})$, $PNM(S_{1}, S_{2})$ and $PNM_{s}(S_{1}, S_{2})$ and those for proto-natural correspondences by $POC(S_{1}, S_{2})$, $POC_{s}(S_{1},  S_{2})$, $POM(S_{1}, S_{2})$ and $POM_{s}(S_{1}, S_{2})$ respectively. If $S_{1}= S_{2}  =  S$, then the notation will be simplified to $SNC(S)$ and the like. For simplicity, we will restrict ourselves to the cases with just one lower and one upper approximation operations on $S_{1}$ and $S_{2}$ in all that follows.
The \textsf{RYS} corresponding to classical \textsf{RST} will be denoted by $R$. We will also use $\mathfrak{C}$ to denote any one the sets of maps above.

\begin{theorem}
On the set $SNC(S_{1}, S_{2})$ and on each of the sets of maps defined above, we can define an induced order via
$\varphi  \leq  \sigma $ iff $(\forall x \in S_{1}) \varphi (x)\subseteq  \sigma (x) $.
This extends to all other sets of maps defined above.
It also extends to all other cases where $S_{2}$ has a partial order on it.
\end{theorem}

In general other induced point-wise operations may not be uniquely definable always in any unique sense without additional constraints. The exceptions are stated after the following theorem.
 
\begin{proposition}
On $\mathfrak{C}$, we can define following the partial operations
\begin{itemize}
\item {For $f,  g\in \mathfrak{C}$, $f \oplus  g = h\in \mathfrak{C} $ iff $h\in \mathfrak{C} $ and \[(\forall x\in S_{1}) h(x)= f(x) \oplus_{2}  g(x).\]}
\item {For $f,  g\in \mathfrak{C}$, $f \odot  g = h\in \mathfrak{C} $ iff $h\in \mathfrak{C} $ and \[(\forall x\in S_{1}) h(x)= f(x) \odot_{2}  g(x) .\]}
\item {For $f, \in \mathfrak{C}$, $\sim f= h\in \mathfrak{C} $ iff $h\in \mathfrak{C} $ and \[(\forall x\in S_{1}) h(x)=\sim f(x) .\]}
\end{itemize}
A similar point-wise definition of lower and upper approximation operations is not possible.
\end{proposition}

\begin{theorem}
On each of $POC(S_{1}, S_{2})$, $POC_{s}(S_{1},  S_{2})$, $POM(S_{1}, S_{2})$ and $POM_{s}(S_{1}, S_{2})$, the following are admissible:
\begin{enumerate}
\item {For $f,  g\in \mathfrak{C}$, $f \oplus  g = h\in \mathfrak{C} $ iff \[(\forall x\in S_{1}) h(x)= f(x) \oplus_{2}  g(x) .\]}
\item {For $f,  g\in \mathfrak{C}$, $f \odot  g = h\in \mathfrak{C} $ iff  \[(\forall x\in S_{1}) h(x)= f(x) \odot_{2}  g(x) .\]}
\item {For $f, \in \mathfrak{C}$, $\sim f= h\in \mathfrak{C} $ iff  $(\forall x\in S_{1}) h(x)=\sim f(x)$. }
\item {For $f, \in \mathfrak{C}$, $ f^{l}= h\in \mathfrak{C} $ iff  $(\forall x\in S_{1}) h(x)=(f(x))^{l} $.}
\item {For $f, \in \mathfrak{C}$, $ f^{u}= h\in \mathfrak{C} $ iff  $(\forall x\in S_{1}) h(x)=(f(x))^{u} $.}
\end{enumerate}
\end{theorem}

\begin{proof}
The proof consists in verifying that $h$ in each of the cases does indeed belong to $\mathfrak{C}$. 
For $POC(S_{1}, S_{2})$, the following holds:
there is a term function $t$ in the signature of $S_2$ such that 
$(\forall x \in \mathcal{G}_1)(\exists y_{1},\ldots y_{n} \in \mathcal{G}_2)
  f(x)=t (y_{1}, \ldots , y_{n}). $
 
\end{proof}

\subsection{Relation to real-valued measures of RST}

We consider the relation to the concept of degree of rough inclusion in classical \textsf{RST} first. Suppose $S_{1},  S_{2}$ are two \textsf{RYS} corresponding to classical \textsf{RST} and let $k_1 ,   k_2 $ be the respective rough inclusion functions on them respectively. For any two elements $X,  Y\in S_{1}$, these are computed according to  
\[k_{1}(X, Y)=
\left\{
\begin{array}{ll}
\dfrac{\#(X\cap Y)}{\#(X)},  & \mathrm{if}  \, X\neq \emptyset, \\
1, & \mathrm{else,}
\end{array}
\right.\]

If $f\in POC(S_{1}, S_{2})$, then we can say very little about $k_{2}(f(X),  f(Y))$ from the value of $k_{1}(X,  Y)$ or conversely.
If the size of all granules involved and their occurrences and the term functions involved in the representation are known then we can possibly actualize some ordering. The converse question is worse. Examples are quite easy to construct for this. Even if $S_{1}= S_{2}$ and granules are related by the identity function, there is no definite connection as the possible values of $f(Z)$, when $Z$ is a non-definite element are not restricted in any way. The situation for $SNC(S_{1}, S_{2})$ is similar to that of $POC(S_{1}, S_{2})$. These aspects transform radically when we restrict the algebraic considerations to collections of approximations or definite elements.

\begin{theorem}
If $S_{1},  S_{2}$ are \textsf{RYS} corresponding to classical \textsf{RST} and $f \in POC(S_{1}, S_{2}) $, then there exists a term function $h$ such that $(\forall B\in \delta (S_{1})(\exists G_{1},  \ldots G_{k}\in \mathcal{G}_{2} )  h(G_{1},  \ldots ,   G_{k})  =  f(B) )$. Further, we may be able to classify such term functions as decreasing, increasing or indefinite relative the relation between the measure functions $k_{1},  k_{2}$. 
\end{theorem}

\begin{proof}
It is clear that if $B\in \delta (S_{1})$ then $(\exists H_{1}, \ldots ,   H_{r} \in \mathcal{G}_{1} \bigcup H_{i} = B  $.
For each of these $H_{i}$, there is a term $t$ such that $f(H_{i})= t(P_{1},   \dots , P_{b})$, with $P_{i}\in\mathcal{G_{2}}$. But the term on the right hand side must be a definite element because of the admissible operations on the \textsf{RYS} and so must be a union of granules.

Because of this we have, $(\forall B\in \delta (S_{1})(\exists G_{1},  \ldots G_{k}\in \mathcal{G}_{2} )  h(G_{1},  \ldots ,   G_{k})  =  f(B) )$. 

\end{proof}
      
\begin{theorem}
In the above theorem, if we modify the conditions as per 
\begin{itemize}
\item {$S_{2}$ is a \textsf{RYS} corresponding to a tolerance approximation space and}
\item {$f\in SNC(S_{1},  S_{2})$,}
\end{itemize}
then the result fails to hold in many situations as term functions acting on granules can yield non-definite elements.. 
\end{theorem}

A simple morphism between two \textsf{RYS} need not preserve granules or definite elements. So $f$ is a morphism that satisfies no other condition then the first conclusion of the first theorem need not necessarily follow. We show this below:

\begin{proposition}
If $S_{1},  S_{2}$ are \textsf{RYS} corresponding to classical \textsf{RST} and $f \in Mor(S_{1}, S_{2}) $, then there need not exist a term function $h$ such that $(\forall B\in \delta (S_{1})(\exists G_{1},  \ldots G_{k}\in \mathcal{G}_{2} )  h(G_{1},  \ldots ,   G_{k})  =  f(B) )$.
\end{proposition}

\begin{proof}
We construct the required counter-example below: 

Let $X_{1}=\{x_{1}, x_{2}, x_{3}, x_{4}\}$ and let the equivalence $Q$ be generated on it by $\{(x_{1}, x_{2}), (x_{2},x_{3})\}$. Taking the granules to be the set of $Q$-related elements, we have \[\mathcal{G}_{1}=\{(x_{1}:x_{2}, x_{3}), (x_{2}:x_{1},x_{3}), (x_{3}:x_{2}, x_{1}), (x_{4}:)\}.\] Here $(x_{1}:x_{2}, x_{1})$ means the successor neighborhood (granule) generated by $x_{1}$ is $(x_{1}, x_{2},  x_{3})$. 

Let $X_{2}=\{a_{1}, a_{2}, a_{3}, a_{4},  a_{5}\}$ and let the equivalence $R$ be generated on it by $\{(a_{1}, a_{4}), (a_{4},a_{2})\}$. Taking the granules to be the $R$-classes, we have $\mathcal{G}_{2}=\{(a_{1}:a_{2}, a_{4}), (a_{2}:a_{1},a_{4}), (a_{3}:), (a_{4}:a_{1}, a_{2}),  (a_{5}:)\}$. 

Let $S_{1},  S_{2}$ be the \textsf{RYS} on the power sets $\wp (X_{1},  \wp (X_{2})$ respectively. 

If $\sigma :  S_{1} \longmapsto  S_{2}$ is a morphism satisfying $\varphi (\{x_1 \})= \{a_{1} \}$, $\varphi (\{x_2 \})= \{a_{1}\}$, $\varphi (\{x_3 \}= \{a_{3}\})$ and $\varphi (\{x_4 \}= \{a_{4} \})$. Under the conditions $\sigma$ is an morphism that is such that the class $(x_{1}: x_{2}, x_{3})$ is mapped to $\{a_{1},  a_{3}\}$, but the latter is not representable in terms of the other granules using $\oplus ,  \odot$ and even complementation.  
\end{proof}

\begin{theorem}
But in general as morphisms need to preserve the parthood (corresponding to inclusion or union), we have in the above context,
$(\exists \alpha , \beta \in \mbR )(\forall X,  Y \in_{f} S_{1}) \alpha k_{1}(X,   Y)  \leq   k_{2}(\sigma (X),   \sigma (Y))  \leq  \beta k_{2} (X,  Y)$. $\in_{f}$ means 'finite element of'. It is necessary that the greatest $\alpha$ and least $\beta$ must exist. 
\end{theorem}

Next we look at elements of $SM_{s}(S_{1},  S_{2})$.

\begin{theorem}
If $S_{1},  S_{2}$ are \textsf{RYS} corresponding to classical \textsf{RST} and $f \in SM_{s}(S_{1}, S_{2}) $, then 
$(\forall B\in \delta (S_{1})(\exists G_{1},  \ldots G_{k}\in \mathcal{G}_{2} )  \bigcup(G_{1},  \ldots ,   G_{k})  =  f(B) )$.
The following condition need not hold even with the additional requirement of $\#(S_{1})= \#(S_{2})$: 
$(\exists \alpha \in \mbR  )(\forall X,  Y\in S_{1})  k_{1}(X,  Y)= \alpha k_{2}(f(X),  f(Y))$.
\end{theorem}

\begin{proof}
Definite elements are unions of granules in $S_{2}$. So from the previous considerations it follows that 
$(\forall B\in \delta (S_{1})(\exists G_{1},  \ldots G_{k}\in \mathcal{G}_{2} )  \bigcup(G_{1},  \ldots ,   G_{k})  =  f(B) )$.
For the second part, all we need to do is to require \emph{dissimilar size and number} of granules in $S_{1}$ and $S_{2}$. Counterexamples are not hard.
\end{proof}

The converse question on the second condition in the above theorem with no assumptions on the nature of $f$ is direction-less.

\begin{theorem}
If $S_{1},  S_{2}$ are \textsf{RYS} corresponding to classical \textsf{RST} and $f \in PNM(S_{1}, S_{2}) $, then there exists a term function $h$ such that \[(\forall B\in \delta (S_{1})(\exists G_{1},  \ldots G_{k}\in \mathcal{G}_{2} )  h(G_{1},  \ldots ,   G_{k})  =  f(B) ).\]
Further, we may be able to classify such term functions as decreasing, increasing or indefinite relative the relation between the measure functions $k_{1},  k_{2}$. 
\end{theorem}
 
\begin{proof}
Since $f\in Mor(S_{1},  S_{2})$, so any union of granules will be mapped to a union of images of granules.
But each image of a granule must be represented by a term function acting on a set of granules in $S_{2}$. As compositions of terms are terms, it follows that $(\forall B\in \delta (S_{1})(\exists G_{1},  \ldots G_{k}\in \mathcal{G}_{2} )  h(G_{1},  \ldots ,   G_{k})  =  f(B) )$.
\end{proof}

\subsection{Putting it Together}

Given the nature of concepts introduced, we can expect some weak connections between nature of 'growth' of correspondences and their type. Specifically these can be about monotonicity being induced generally or on a quotient. This will useful for simplifying the theory and applications. Here we consider a few specific cases alone. A more thorough investigation will be part of future work.

The first theorem concerns self-maps. 

\begin{theorem}
If $S$ is the \textsf{RYS} corresponding to classical \textsf{RST}, $f\in SM_{s}(S)$, $g\in SNC(S)$ and $g \in \Theta_{lu} (f)$, then
there is a filter $H$ of $S$ such that \[(\forall x\in \delta(S)\cap H)   g(x)=f(x).\]
\end{theorem}

\begin{proof}
Suppose $x\in \delta(S)$, then $(\exists z_i \in \mathcal{G})  \bigcup z_{i}= x$.
Suppose $z_o$ is a fixed element in $\S_{c}$ and $(\forall z_{o}\subset z)  f(z)^l \subseteq  g(z)  subseteq  f(z)^u$.
If $z_{o} \subset  x$, then $f(x)^l= (\bigcup f(z_{i}))^{l}  =  \bigcup f(z_{i}) \subseteq  g(\bigcup z_{i}) subseteq \bigcup f(z_{i})$. 

So $(\forall  x \in   \delta(S)\cap z_{o}\uparrow$, we have $g(x)= f(x)$. 
\end{proof}

\begin{proposition}
If $S_{1},  S_{2}$ are \textsf{RYS} corresponding to classical \textsf{RST} and $f,  g \in SM_{s}(S_{1}, S_{2}) $ and $g \in \Theta_{lu} (f)$, then there exists a congruence $\rho$ on $S_{1}$ such that the induced quotient morphisms $[f],  [g]$ coincide on $\delta(S_{1}|\rho )$.
\end{proposition}

\begin{theorem}
If $S_1,  S_2$ are arbitrary lattice ordered \textsf{RYS} with the operations $\oplus, \odot$  corresponding to the lattice orders $\pc_2$ on $S_2$ and $f\in SM_{s}(S_1 ,  S_2)$, then on $\Theta_{lu}(f)\cap SM_{s}(S_{1}, S_{2})$, the following point-wise operations are well defined (for simplicity we will assume a single pair of lower and upper approximation operators):
\begin{itemize}
\item {$(\forall g, h)(\forall x\in S_{1}) (g\oplus h)(x)= g(x)\oplus  h(x)$.}
\item {$(\forall g, h)(\forall x\in S_{1}) (g\odot h)(x)= g(x)\odot  h(x)$.}
\item {$(\forall h)(\forall x\in S_{1}) (h^L)(x)= (h(x))^{l}  \&  (h^{U})(x)= (h(x))^{u} $.}
\end{itemize}
\end{theorem}

\begin{proof}
We have $(\exists z_o\in S_{1c} )(\forall z) (\pc_{1}z_{o} z \longrightarrow  \pc_{2}f(z)^{l} g(z) \& \pc_{2} g(z)f(z)^u$ and
$(\exists z_1\in S_{1c} )(\forall z) (\pc_{1}z_{1} z \longrightarrow  \pc_{2}f(z)^{l} h(z) \& \pc_{2} h(z)f(z)^u$. 
As $\pc_2$ is a lattice order, we can definitely conclude that  $(\forall z) (\pc_{1}(z_{o}\vee z_{1}) z \longrightarrow  \pc_{2}f(z)^{l} (g(z)\oplus h(z)) \& \pc_{2} (g(z)\oplus h(z))f(z)^u$.

Similarly the other parts can be proved.
 
\end{proof}

\begin{theorem}
In the above theorem, we can replace $SM_{s}(S_{1},S_{2})$ uniformly with $SM(S_{1},S_{2})$.   
\end{theorem}

\subsection{Extended Example}

Let $X_{1}=\{x_{1}, x_{2}, x_{3}, x_{4}\}$ and let the tolerance $T$ be generated on it by $\{(x_{1}, x_{2}), (x_{2},x_{3})\}$. Taking the granules to be the set of $T$-related elements, we have $\mathcal{G}_{1}=\{(x_{1}:x_{2}), (x_{2}:x_{1},x_{3}), (x_{3}:x_{2}), (x_{4}:)\}$. Here $(x_{1}:x_{2})$ means the successor neighborhood (granule) generated by $x_{1}$ is $(x_{1}, x_{2})$. 

Let $X_{2}=\{a_{1}, a_{2}, a_{3}, a_{4},  a_{5}\}$ and let the equivalence $R$ be generated on it by $\{(a_{1}, a_{4}), (a_{4},a_{2})\}$. Taking the granules to be the $R$-classes, we have $\mathcal{G}_{2}=\{(a_{1}:a_{2}, a_{4}), (a_{2}:a_{1},a_{4}), (a_{3}:), (a_{4}:a_{1}, a_{2}),  (a_{5}:)\}$. 

Let $S_{1},  S_{2}$ be the \textsf{RYS} on the power sets $\wp (X_{1}),  \wp (X_{2})$ respectively. If $\varphi :  S_{1} \longmapsto  S_{2}$ is a injective map satisfying $\varphi (\{x_1,  x_2 \})= \{a_{1},  a_{2},  a_{4}\}$, $\varphi (\{x_2,  x_1 , x_{3}\})= \{a_{2},  a_{1},  a_{4}\} \cup  \{a_{3}\}$, $\varphi (\{x_3,  x_2 \}= \{a_{5}\})$ and $\varphi (\{x_4 \}= \{a_{3} \})$. Then $\varphi$ is an example of a SNC that cannot be a $\oplus$-morphism. 

Let $\sigma :  S_{1} \longmapsto  S_{2}$ be a $\oplus$-morphism satisfying $\sigma (\{x_1 \})= \{a_{1} \}$, $\sigma (\{x_2 \})= \{a_{2}\}$, $\sigma (\{x_3 \}= \{a_{3}\})$ and $\sigma (\{x_4 \}= \{a_{4} \})$.

\begin{proposition}
$\sigma$ is an injective $\oplus$-morphism that fails to be a sub-natural correspondence.  
\end{proposition}
\begin{proof}
The class $(x_{1}: x_{2})$ is mapped to $\{a_{1}, a_{2}\}$, but the latter is not representable in terms of the other granules using $\oplus ,  \odot$.  
\end{proof}

If $\tau :  S_{1} \longmapsto S_{2}$ is a map satisfying $\tau (\{x_1 \})= \{a_{3} \}$, $\tau (\{x_2 \})= \{a_{5}\}$, $\tau (\{x_3 \}= \{a_{3}\})$ and $\tau (\{x_4 \}= \{a_{5} \})$, then $\tau$ is not injective on granules and is a proto-natural correspondence. Some of these proto-natural correspondences are also morphisms. It is possible to define a number of PNCs that are greater than $\varphi$ by minimal modification of $\varphi$. For example, we can add an extra granule to $\varphi (\{x_{3}\})$.  

\section{Contamination and Classical RST }

The \emph{requirements of a contamination-free semantics} at Meta-R for classical \textsf{RST} may seen to be 
\begin{itemize}
\item {The objects of interest are roughly equivalent sets.}
\item {The operations used in the semantics are as contamination-free as is possible.}
\item {The logical constants in the associated logic are as real (or actualizable) as is possible.}
\end{itemize}
The last two criteria are very closely related and one may be expected to determine the other. The first of the three criteria is fairly clear, but the second and third are relative the meaning in the intended use of the semantics. 

The natural way of realizing the contamination of operations relative basic operations would be through some concept of definability or representability. Taking orders on rough objects as basic predicates, we can for example regard $\sqcup$ as a non-contaminated operation in pre-rough/rough algebras (as it is definable). From the point of view of representation as a term, $\sqcup$ would be contaminated as higher order constructions would be required. It is also possible to regard the pre-rough/rough algebra or equivalent semantics as being essentially over-determined and so the problem would be of weakening the semantics. Key properties that determine the last two requirements relate to level of perception of rough inclusion and bigness. $\sqcup, \sqcap$ may be basic operations or these may not be due to constructive limitations or the extraction of least upper and greatest lower bound is done in a sloppy fashion (lazy order). This is very important in modeling human reasoning. In many contexts the bounds may be dependent on the relative bigness or otherwise of the outcome of the specific instance of $\sqcup$ or $\sqcap$. The bigness based cases are not about over-determination of the problem and can be associated with filters, ideals and intervals (or generalisations thereof) of different types in most cases and then would be semantically amenable. 

\subsection{Relevant Big Rough Algebras}

If $S$ is a finite approximation space (finiteness can be relaxed), and if $\wp (S)|\approx$ is the poset of roughly equal objects ordered by rough inclusion $\sqsubseteq$, then we know that the operations $\sqcup,  \sqcap$ are definable in it by way of $\sqsubseteq$ being a lattice order.
But a definition of these by terms would not be possible over the pre-rough algebraic system $\left\langle \wp (S)|\approx ,  \leq ,  L,   \neg ,    0,   1\rlan $ (the superfluous operations over pre-rough algebras \cite{BC2} are omitted). Analogous considerations apply to other variants of aggregation and commonality in classical \textsf{RST}. Missing proofs will appear separately.

\begin{proposition}
In the theory of finite classical \textsf{RST-RYS}, it is possible to define the interpretation of the operations $\sqcup $ and $\sqcap$ over $\wp (S)|\approx$. 
\end{proposition}

\begin{proof}
Since $\approx$ is a derived predicate, the representations of the operations can be carried over.
\end{proof}

The following concepts of filters and ideals capture the concept of closure under types aggregation and commonality operations and consequence operators. Importantly some filters and ideals can be regarded as sufficiently big or not big subdomains.

\begin{definition}
An arbitrary subset $K$ of $\wp(S)|\approx = Q$ will be said to be a \emph{L-Filter} iff it satisfies F0 and O1. If in addition it satisfies F1, then it will said to be prime. $K$ will be an o-filter if it satisfies F0 alone :
\begin{itemize}
\item {F0: $(\forall x\in K)(\forall y\in Q)(x \leq y \Rightarrow y\in K) $.}
\item {O1: $(\forall x \in K) Lx \in K $.}
\item {F1: $(\forall a, b \in Q)(1\neq a\sqcup b\in K \Rightarrow  a\in K
\;\mathrm{or}\; b\in K) $.}
\end{itemize}
The dual notions will be that of \emph{U-Ideals, prime  U-ideals and o-ideals}. If a L-filter is closed under $\sqcap, \sqcup$, then it will be termed a lattice L-filter.
Let $\mathbb{K}  =  \llan{K}, \leq, L,  U,  \neg,  1 \rlan $ be the induced partial algebraic system on $K$.
\end{definition}

\begin{proposition}
If $K$ is a lattice L-filter, then $\mathbb{K}$ is not a pre-rough algebra, but satisfies:
\begin{enumerate}
\item {$\leq$ is a distributive lattice order.}
\item {Closure under $L,   U$, but not under $\neg$. }
\item {$Lx \leq x$ ; $L(a\sqcap b)  =  La\sqcap Lb$; $LLx = Lx$;}
\item {$L1 = 1$; $ULx =Lx$; $L(a\sqcup b)=La\sqcup Lb$}
\end{enumerate}
\end{proposition}

\begin{proof}
If we assume finiteness, then the lattice is bounded. But we would have no way (in general) of ensuring closure under $\neg,   0$.
The three element pre-rough algebra provides the required counterexample.    
\end{proof}

\begin{proposition}
If $K$ is a L-filter, then $\mathbb{K}$ satisfies:
\begin{enumerate}
\item {$\leq$ is a join-semilattice lattice order ($\sqcup$ is definable).}
\item {Closure under $L,   U$, but not under the partial lattice operation $\sqcap$ and $\neg$. }
\item {$Lx \leq x$ ; $L(a\sqcap b) \stackrel{w}{=}  La\sqcap Lb$; $LLx = Lx$;}
\item {$L1 = 1$; $ULx =Lx$; $L(a\sqcup b)=La\sqcup Lb$ ; $x\sqcup (y\sqcap x) \stackrel{w}{=} x $.}
\item {$x\sqcup (y\sqcap z) \stackrel{w}{=}  (x\sqcup y)\sqcap (x\sqcup z)$ and its dual.}
\end{enumerate}
\end{proposition}

\begin{proof}
If $a\sqcap b\in K$, then $L(a\sqcap b)\in K$ by definition and so $La,  Lb,  a,   b \in K $.
If $La, Lb \in K$, then it is possible that $La\sqcap Lb\notin K$, which is the reason for the weak equality.  
\end{proof}

\begin{theorem}
There exists a pre-rough algebra $S$ with a nontrivial lattice L-filter $K$ satisfying 
\[(\exists a, b\in S\setminus \{1\} )(\forall c \in K\setminus \{1\})   a\sqcup b \shortparallel c .\]   
\end{theorem}

The proof involves a simple construction, but the point we want to make is that $K\setminus \{1\}$ may or may not be cofinal  
in $S\setminus \{1\}$. This is important as such a $K$ may be interpreted to consist of \emph{big} elements alone. We will refer to such L-filters or lattice L-filters as \emph{cofine}.

\begin{theorem}
Given a pre-rough algebra with no nontrivial lattice L-filters, we can construct an infinite number of pre-rough algebras with the same property. 
\end{theorem}

\begin{proof}
We suggest a completely visual proof for this. Simply paste a pair of three element pre-rough algebras to the original pre-rough algebra (identifying all the tops and bottoms respectively) and require that the negation of one of the non boundary element is the other. The infinite number of pre-rough algebras follow by recursive application of the process. 

A second proof can be through the fact that the product of two pre-rough algebras with the property satisfies the property.      
\end{proof}

\begin{definition}
Given a L-filter $K$ on $Q$, for any $x,  y\in Q$ let 
\begin{equation*}
x \Cup  y  = 
\begin{cases}
 x\sqcup y & \text{if} \;  x\sqcup y \in K\\
\text{undefined} & \text{otherwise.}
\end{cases}
\end{equation*}
\begin{equation*}
x \Cap  y  = 
\begin{cases}
 x\sqcap y & \text{if}\;   x\sqcap y \in K\\
\text{undefined} & \text{otherwise.}
\end{cases}
\end{equation*}
Further let $x\lhd y$ iff $x=y$ or $x\Cap y = x$ or $x\Cup y=y$.
\end{definition}

\begin{proposition}
The relation $\lhd$ is a partial order that is not necessarily a lattice order, but is compatible with the operations $L, U$. Further the restriction of $\lhd$ to $K$ has already been described above. 
\end{proposition}
\begin{proof}
Absorption laws can be shown to fail in most pre-rough algebras for a suitable choice of a L-filter. 
\end{proof}

\begin{definition}
By a \emph{operationally contamination-free prerough} algebraic system (or OCPR system), we will mean a partial algebraic system of the form 
$Y= \left\langle \underline{Q},  \lhd ,  L ,   U,  \Cup,  \Cap ,   0,  1\right\rangle $, with the operations and relations being as defined above ($U$ is the operation induced by upper approximation operator on $Q$). 
\end{definition}

\begin{definition}
By a \emph{OC-system} (resp \emph{lattice OC-system}) we will mean a pair of the form $\llan{Q}, K \rlan$ consisting of a pre-rough algebra $Q$ and a L-filter (resp. lattice L-filter ) $K$. If $K$ is cofine, then we will refer to the system as being \emph{cofine}.  
\end{definition}

\begin{theorem}
\begin{enumerate}
\item {If $K$ is a lattice L-filter, then $K^{+}= \{y : :(\forall x\in K) x\sqcup y =1 \}$ with induced operations from the pre-rough algebra is a lattice L-filter. Such filters will be termed \emph{supremal}.}
\item {$K$ is a cofine lattice L-filter iff $K^{+}  =  \{1\}$.}
\item {The collection of all supremal lattice L-filters can be boolean ordered with an order distinct from the order on lattice L-filters.}
\end{enumerate}
\end{theorem}

Thus starting from a standard rough domain (corresponding to pre-rough algebras), we have arrived at new rough semantic domains.
At least two distinct partial algebras can be defined with one being an extension of a pre-rough algebra, while OC-pre-rough systems constitute a severe generalization. The natural correspondences from a pre-rough algebra to a cofine L-filter (or lattice L-filter) would be forgetful closed morphisms that preserves all operations except for $\neg$.

\subsection*{Remarks} 

In this research paper, we have developed the mathematics of fine-grained comparison of one \textsf{RYS} with another and of reducing contamination of operations. The process has involved a number of steps including the exactification of the concept of granular rough evolution, identification of various types of correspondences, concepts of comparison of those correspondences and algebraisation of concepts of relevance/bigness. These constitute an alternative/ supplement to the earlier approach to measures due to the present author in \cite{AM240} and of course the usual real-valued measures of \textsf{RST}. The number of types of correspondences has also been expanded upon relative \cite{AM1800} and more properties have been established. As per the new approach comparison with classical rough set semantics or any other rough semantics should essentially be constructed by the force of granular axioms. 

\emph{This Paper is a peprint of the paper presented in FUZZIEEE'2013 and is also available in IEEE Xplore:10.1109/FUZZ-IEEE.2013.6622521 }

\bibliographystyle{IEEEtran}
\bibliography{biblioam06092015.bib}
\end{document}